\documentclass[10pt,twocolumn,letterpaper]{article}

\usepackage{iccv}
\usepackage{times}
\usepackage{epsfig}
\usepackage{graphicx}
\usepackage{amsmath}
\usepackage{amssymb}
\usepackage{booktabs}
\usepackage{fullpage}
\usepackage{multirow}
\usepackage{algorithmic}
\usepackage{subfigure}
\usepackage{mathrsfs}
\usepackage{multirow}
\usepackage{rotating}
\usepackage{amsmath}
\usepackage{amsthm}
\usepackage{arydshln}
\usepackage{booktabs}
\usepackage{algorithm}
\usepackage{graphicx}
\usepackage{subfigure}
\usepackage{cuted}
\usepackage{authblk}
\usepackage{fancyhdr,graphicx,amsmath,amssymb}
\usepackage[dvipsnames]{xcolor}
\usepackage{lineno}

\usepackage[pagebackref=true,breaklinks=true,letterpaper=true,colorlinks,bookmarks=false]{hyperref}

\iccvfinalcopy 

\ificcvfinal\fi

\begin{document}

\title{Learning Bias-Invariant Representation by Cross-Sample Mutual Information Minimization}

\author[1]{Wei Zhu}
\author[1]{Haitian Zheng}
\author[2]{Haofu Liao \thanks{This work was done when Haofu Liao was at the University of Rochester.}}
\author[1]{Weijian Li}
\author[1]{Jiebo Luo}
\affil[1]{University of Rochester}
\affil[2]{Amazon Web Services}
\affil[ ]{\textit {\{wzhu15,hzheng15,wli69\}@ur.rochester.edu, liaohaofu@gmail.com, jluo@cs.rochester.edu}}

\maketitle

\ificcvfinal\thispagestyle{empty}\fi

\begin{abstract}
Deep learning algorithms mine knowledge from the training data and thus would likely inherit the dataset's bias information. As a result, the obtained model would generalize poorly and even mislead the decision process in real-life applications. We propose to remove the bias information misused by the target task with a cross-sample adversarial debiasing (CSAD) method. CSAD explicitly extracts target and bias features disentangled from the latent representation generated by a feature extractor and then learns to discover and remove the correlation between the target and bias features. The correlation measurement plays a critical role in adversarial debiasing and is conducted by a cross-sample neural mutual information estimator. Moreover, we propose joint content and local structural representation learning to boost mutual information estimation for better performance. We conduct thorough experiments on publicly available datasets to validate the advantages of the proposed method over state-of-the-art approaches. 
\end{abstract}
\section{Introduction}

\begin{figure}
\centering
\includegraphics[width=0.5\textwidth]{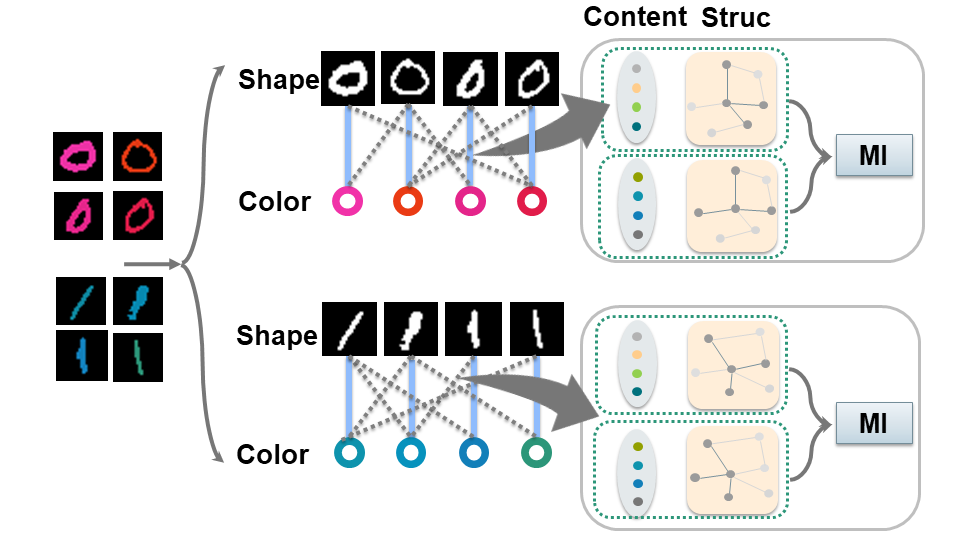}
\caption{A brief illustration of CSAD for a color-biased binary classification task. Our objective is to obtain a color-invariant digit classifier. Given the $i$-th training sample $x_i$ as a red ``0", most existing methods eliminate the correlation between $h_i$ and red information from the $i$-th sample (solid line), where $h_i$ is the extracted features. By contrast, CSAD could reduce the correlation between $h_i$ and various red colors extracted from other samples. A cross-sample mutual information estimator measures the correlation with joint content and local structural representation. }
\label{fig:demo}
\end{figure}

Modern machine learning is built on collected and contributed data. However, real-world data inevitably contains noise and bias and may not be well-distributed. Such flawed datasets may make the learned model unreliable and pose threats to the learned model's generalization capacity to unseen data. This problem is particularly crucial for medical and healthcare-related applications~\cite{chen2019developing}. For example, Parkinson's Disease (PD) is associated with age, and the PD patients are primarily composed of older people in the related datasets~\cite{schwab2019phonemd,bot2016mpower,li2020predicting}. A model learned on these datasets may predict PD by the age of patients instead of the symptoms of the disease. As a result, the age bias renders the learned model hardly useful for real-life disease diagnosis and analysis.

Several methods have been proposed to learn to remove the dataset bias~\cite{kim2019learning,alvi2018turning,bsdfhgdgfh,zhang2018mitigating,nam2020learning,bahng2020learning}. Among them, some methods regularize the model to not learn bias with additional regularization terms~\cite{nam2020learning,bahng2020learning}, and others learn to eliminate the learned bias information by adversarial learning \cite{kim2019learning, bsdfhgdgfh,zhang2018mitigating}. Our work follows the latter and the bias elimination is often conducted by minimizing the correlation between the extracted features and the bias label. The correlation measurement plays a critical role in adversarial debiasing and is often fulfilled by recently proposed neural mutual information estimators \cite{ge2020mutual,belghazi2018mutual,DBLP:conf/iclr/HjelmFLGBTB19}. In particular, \cite{kim2019learning} proposes to adversarially discover and remove the bias information by adding a gradient reversal layer between the feature extractor and the bias branch. \cite{bsdfhgdgfh} adversarially learns to mitigate the bias by minimizing the mutual information between the latent representation and the bias label. To conclude, they essentially learn to remove the bias information from the target classifier by eliminating the dependency between the target and the bias information from same training sample.   
And as a result, these methods are limited to model and reduce the correlation within each training sample and totally neglect the rich cross-sample information. However, we note that the cross-sample information is important and necessary to be taken into consideration for debiasing.  For example, as shown in Fig. (\ref{fig:demo}), given the $i$-th sample $x_i$ as a red ``0", it is not enough to only eliminate the correlation between $h_i$ and the red representation extracted from $x_i$, as the correlation between $h_i$ and the color representation extracted from other red digits will be preserved, where $h_i$ is the representation of $x_i$. That is, $h_i$ may still be highly biased and is correlated to pink, rose, ruby, etc. The neglect will pose a grave threat to the reliability of the learned correlation measurement and eventually leads to sub-optimal performance for debiasing. Moreover, although the local structural representation is proven to be helpful for correlation learning~\cite{zhang2019adaptive,aziere2019ensemble}, it is also hard to be incorporated with existing methods~\cite{kim2019learning,zhang2018mitigating,alvi2018turning}.


To address the issues as mentioned above, we propose a cross-sample adversarial debiasing (CSAD) method as shown in Fig. \ref{fig:demo}. To make it possible to utilize the cross-sample and structural information, inspired by recent progress on domain adaptation~\cite{peng2019federated,peng2019domain}, CSAD first explicitly disentangles the target and the bias representation. Then, CSAD relies on a cross-sample neural mutual information estimator for correlation measurement, which is conducted on the disentangled bias and target representation. This could also avoid potential problems caused by the domain gap between the latent representation and the bias label used by other methods \cite{kim2019learning,bsdfhgdgfh}. With the cross-sample information, CSAD could comprehensively eliminate correlation between target and bias information from different samples. Additionally, the explicit disentanglement makes it possible to consider the local structural representation for mutual information estimation. Specifically, we encourage the bias and target representation to have different topological structures captured by Random Walk with Restart \cite{tong2006fast}, which could avoid the bias information of certain sample to be guessed by its neighbors.


We highlight our main contributions as follows:
\begin{enumerate}
    \item We propose a flexible and general framework for adversarial debiasing that can explicitly disentangle target and bias representation.
    \item Based on the proposed framework, we propose cross-sample adversarial debiasing (CSAD). CSAD eliminates the bias information by a cross-sample mutual information estimator that can jointly exert cross-sample content and structural features. 
    \item We conduct extensive experiments on benchmark datasets, and our method achieves substantial improvement compared to the current state-of-the-art methods.
\end{enumerate}

\section{Related Work}

\subsection{Debiasing and Fairness}
Biases exist across race, gender, and age, and they pose threats to machine learning models in diverse tasks, such as image classification~\cite{geirhos2018imagenet,robinson2020face,dhar2020adversarial,choi2019can,wang2019learning}, representation learning~\cite{lahoti2019operationalizing,lahoti2019ifair,moyer2018invariant,creager2019flexibly}, word embedding~\cite{zhao2019gender,bolukbasi2016man} and visual question answering~\cite{clark2019don}. A straightforward way to address the problem is to collect~\cite{panda2018contemplating} or synthesize more data to balance the training set~\cite{geirhos2018imagenet,shetty2019not,ray2019sunny}. However, unbiased data could be expensive to collect and impractical to generate for general tasks.  
 
Other methods alleviate the bias through the learning process. Alvi \textit{et al.} avoid to learn the bias by a maximum cross-entropy term~\cite{alvi2018turning}. SenSR adopts a variant of individual fairness as a regularizer so that the learned model could satisfy the individual fairness~\cite{yurochkin2019training}. Similarly, SenSeI achieves the individual fairness through a transport-based regularizer~\cite{yurochkin2020sensei}. Zafar \textit{et al.} develop fairness methods based on the decision boundary fairness~\cite{zafar2017fairness}. DRO regularizes the model by considering worst group performance \cite{sagawa2019distributionally}. Learned-Mixin encourages the model to focus on different patterns with a ensemble framework \cite{clark2019don}. ReBias solves a min-max game to encourage the independence between network and biased prediction \cite{bahng2020learning}. Besides, adversarial learning is also adopted for debiasing, and most methods utilize a discriminator to predict the bias label or to estimate the dependency between latent representation and the bias label. For example, Zhang \textit{et al.} trains the discriminator for bias label with the soft assignment of target \cite{zhang2018mitigating}. Kim \textit{et al.} unlearns the bias extracted by a bias predictor by adopting a gradient reversal layer \cite{kim2019learning}. Ruggero \textit{et al.} learns to minimize the mutual information between the latent representation and the bias labels \cite{bsdfhgdgfh}. 
Few of these methods consider cross-sample and structural information that we discussed above. 

Recently, several works focus on learning with weak or even no bias supervision~\cite{lahoti2020fairness,nam2020learning,wang2019learning,lahoti2020fairness}. Learning from Failure (LfF) puts more weight on failure samples \cite{nam2020learning}. Adversarial Reweighted Learning (ARL) \cite{lahoti2019ifair} adversarially learn a distribution of hard samples. These models may not be robust as they also tend to overfit the noisy samples~\cite{lahoti2019ifair}, and the practical significance still need to be validated.

\subsection{Mutual Information Estimation}
Mutual information is used to measure the dependency between random variables. As the exact value of mutual information is prohibitive to calculate for large scale data, several papers apply neural networks for efficient mutual information estimation \cite{ge2020mutual,belghazi2018mutual}. For random variables $X$ and $Z$, we denote the product of the margin distribution as $P_XP_Z$ and the joint distribution as $P_{XZ}$, and the mutual information between $X$ and $Z$ could be estimated by a neural network $M$ by training it to distinguish the samples drawn from the joint distribution $P_{XZ}$ and those drawn from the product of marginal distribution $P_XP_Z$, e.g., MINE~\cite{belghazi2018mutual}, Deep InfoMax \cite{hjelm2018learning}, etc.
Since we are not interested in the exact value of the mutual information, a lower bound of the mutual information derived from Jensen Shannon Divergence could be formulated as~\cite{hjelm2018learning}
\begin{equation} \label{eq:mijsd}
\begin{split}
    I_{\text{JSD}}(X,Z) = &\sup  E_{P_{XZ}} [-\text{sp}(-M(x,z))] \\ &-E_{P_XP_Z}[\text{sp}(M(x,z))],
\end{split}
\end{equation}
where $\text{sp}(x)=\log(1+\exp(x))$ is the softplus function and $M$ is a neural network.
Information Noise-Contrastive Estimation (Info NCE) \cite{oord2018representation} is also proposed as a normalized mutual information estimator~\cite{poole2019variational}.

\section{Problem Statement}  
Formally, given the $i$-th sample from the training dataset as $(x_i, y_i,b_i)$, where $x_i$ is the input data, $y_i$ is the groundtruth label for the target task, and $b_i$ is the bias label, we first train a target classifier with a feature extractor $F$ that outputs the latent base representation $h_i$ to maximize the performance on the target task. However, since there are correlations between the target task label $y$ and the bias label $b$, the target task is likely to rely on the bias information to fulfill its objective, and as a result, the optimized model would generalize poorly on unseen bias-free data. In this paper, we aim to remove the bias information extracted by $F$ while preserving the performance for the target task. As a example, for the colored MNIST dataset shown Fig.~\ref{fig:demo}, different digits will be class-wisely painted with similar color in the training set, \textit{e.g.}, red for zero and green for one. The vanilla digit classifier would likely learn to predict digits based on its color. The objective of debiasing is to make the representation $h$ contain only shape information with no color information. In other words, the digit classifier would rely on shape instead of color to fulfill the task and thus would finally lead to a more practical model.
\begin{figure}
\centering
\includegraphics[width=0.48\textwidth]{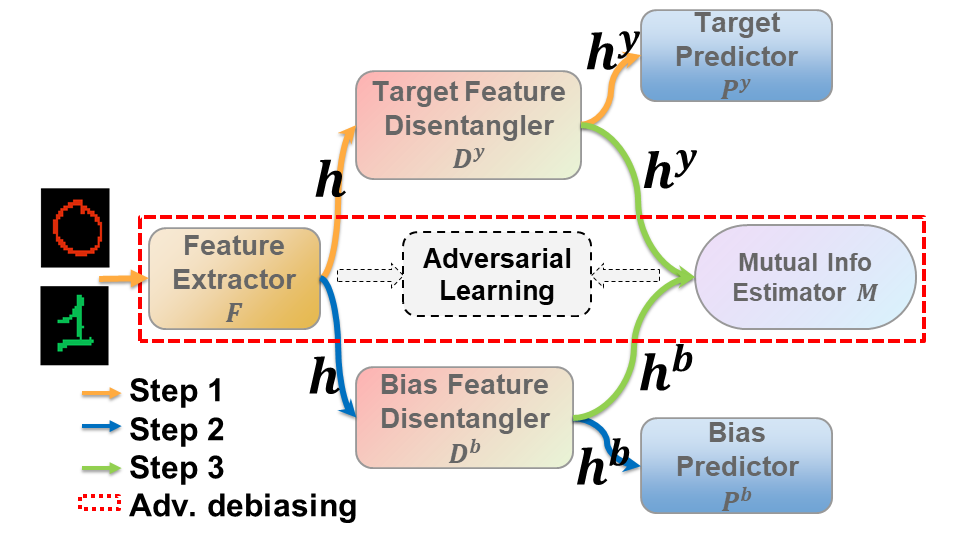}
\caption{Block diagram of CSAD for the Colored MNIST dataset, and our objective is to obtain color-invariant digit classifier. 
}
\label{fig:dig1d}
\end{figure}
\section{Method}
Our method is illustrated in Fig. (\ref{fig:dig1d}) and is detailed in Algorithm (\ref{algdfd}). Our basic idea is to minimize the correlation between the disentangled bias and target representation with a mutual information estimator. We briefly introduce the training procedure as follows: we first pretrain a target classifier composed of $F$, $D^y$, and $P^y$ until convergence. The output of $F$ is denoted as the base representation $h$ and the output of $D^y$ is target representation $h_y$ (Step 1); then, we extract the bias representation $h_b$ from $h$ by a bias classifier composed of $D^b$ and $P^b$. Note we do not update $F$ for bias predication (Step 2); then, we learn the correlation between the bias $h^b$ and the target representations $h^y$ by optimizing the mutual information estimator $M$ (Step 3); at last, we update $F$ to minimize the mutual information estimated by $M$ (Adv. debiasing).

Note that the feature extractor $F$ is only updated to minimize the mutual information estimated by $M$ (Adv. debiasing) and the target prediction loss (Step 1); it is thus forced to generate bias-invariant features that are still powerful enough for the target task. As a result, after training, the latent representation $h$ extracted by $F$ will contain little information on bias. The bias feature disentangler $D^b$ thus cannot extract useful information $h^b$ that could enable mutual information estimator to associate $h^y_i$ with $h^b_j$ where the $i$-th and $j$-th samples share similar bias. The bias branch and mutual information estimator will be discarded at testing time.



\subsection{Cross-Sample Adversarial Debiasing}

One of the most critical part of our method is to measure the correlation between the target $h^y$ (shape) and bias $h^b$ (color) representations which is fulfilled by a neural mutual information estimator. To conduct the neural estimation, we first need to define the joint and product of marginal distributions between $h^y$ (shape) and $h^b$ (color). As discussed in~\cite{tschannen2019mutual}, there is a close relation between mutual information estimation and metric learning. For ease of presentation, we denote the samples drawn from the \textit{joint distribution} $P_{h^y h^b}$ as \textit{positive pairs} and those drawn from the \textit{product of marginal distribution} $P_{h^y}P_{h^b}$ as \textit{negative pairs}. According to recent literature, the pair construction, \textit{i.e.}, the definition of $P_{h^yh^b}$ and $P_{h^y}P_{h^b}$, plays a critical role in mutual information estimation~\cite{bachman2019learning}. Usually, the positive pairs could be intuitively constructed by matching different representations of the same sample, while the negative pairs are generated by matching the representation of different samples. For example, Deep InfoMax constructs the positive pairs by matching the global feature and local feature from the same sample~\cite{DBLP:conf/iclr/HjelmFLGBTB19}; Contrastive Multi-view Coding (CMC) regards different views of the same sample as positive pairs~\cite{tian2019contrastive}; Contrastive Predictive Coding (CPC) exerts the consecutive of sequential data to construct the positive pairs~\cite{oord2018representation}.

Specifically for debiasing, it is common to have the target representation and the bias label of same sample as the positive pairs \cite{kim2019learning,bsdfhgdgfh}. However, given the i-$th$ sample as a red ``0", these methods only try to extract and reduce the dependency between the target representation $h_i^y$ and the color label ``red", and neglect the dependency between $h_i^y$ and other biases that are similar to the ``red" and can be extracted from other samples e.g., pink, rose, light red, and also variants of ``red" from other samples. The debiasing models trained without considering these cross-sample information will then omit the potential target-bias correlation, and eventually lead to suboptimal performance. However, it is hard to utilize the cross-sample relation between data especially with same bias label by existing methods~\cite{kim2019learning,bsdfhgdgfh}, since they directly reduce the mutual information between target representation and bias label. By contrast, CSAD conducts debiasing on feature level and explicitly disentangles the target $h^y$ and the bias $h^b$ representation. Then the joint and product of marginal distribution could be easily defined in an across-sample way for debiasing. That is, the positive pairs for the $i$-th sample are constructed by the target representation $h_i^y$ and bias representation $h^j_b$, where $(i,j)\in \Omega$ is a set of pairs which share similar bias and $i$ is not necessary to be equal to $j$. Moreover, as shown in Section \ref{sec:struc}, the disentanglement framework also makes it possible to consider the structural representation. 

With the definition of the positive and negative pairs, we can conduct neural mutual information estimation by existing estimators e.g., MINE and $I_{JSD}$. However, these estimators are developed under the assumption that there is only one positive pair per sample, and are thus suboptimal for our case. To consider the multiple positive pairs for each sample, inspired by recent progress on metric learning \cite{wang2019multi}, we propose a cross-sample mutual information estimator for adversarial debiasing as follows
\begin{equation} \label{eq:ourmi}
\begin{split}
     I_{\text{CS}}(h^y, h^b) = \sup  &-\log (1+E_{(i,j)\in \Omega} \exp({-M (h^y_i,h^b_j)})) \\
     - & \log (1+ E_{(i,j)\notin \Omega} \exp({M (h^y_i,h^b_j)})),
\end{split}
\end{equation}
where $M (h^y_i,h^b_j)$ is implemented by a neural network detailed in Sec. \ref{sec:feature_learning}, and its output denotes the correlation between $h^y_i$ and $h^b_j$. 

We here draw several points for CSAD. First, comparing with existing methods \cite{kim2019learning,bsdfhgdgfh}, CSAD relies on a disentanglement framework and conducts mutual information estimation on feature level $h^y$ and $h^b$ instead of using the bias label directly. This makes it possible to consider cross-sample correlation and avoid the potential domain gap. Second, the positive pairs (joint distribution) of CSAD are constructed by  $h^y_i$ and $h^b_j$, where $(i,j) \in \Omega$ and $i$ is not necessary to be equal to $j$. Third, the feature-level debiasing framework also makes it possible to take the local structure of each sample into consideration for mutual information estimation as described later in Section \ref{sec:struc}, and jointly adopts content and structural representation yields better performance as shown in our experiments. At last, it is easy to see that the proposed $I_{CS}$ is a lower bound for $I_{JSD}$ shown in Eq.~(\ref{eq:mijsd}). The proof is simple and is based on Jensen's inequality and the convexity of $-\log(1+x)$ \cite{boyd2004convex}. Refer to supplemental material for detail. As claimed by \cite{tschannen2019mutual}, maximizing a tighter bound on mutual information does not always lead to better performance. Compared with $I_{JSD}$, the proposed $I_{CS}$ can automatically reweight training samples~\cite{wang2019multi}, and shows superior performance in practice according to our experiments.

\subsection{Representation Learning for MI Estimation} \label{sec:feature_learning}
This section provides the detail for the implementation of mutual information estimator $M$ with content and local structural representation. 

\subsubsection{Content Representation Learning} \label{sec:content}
We first demonstrate the proposed $I_{CS}$ with content representation. We implement $M$ with two branches as $M =\lbrace M_y, M_b \rbrace$, and pass $h^y$ and $h^b$ through $M_y$ and $M_b$ respectively to obtain $M_y(h^y)$ and $M_b(h^b)$. Then $M(h_i^y, h_j^b)$ is calculated by the cosine similarity between $ M_y(h_i^y)$ and $M_b(h_j^b)$ as
\begin{equation} \label{eq:content}
\begin{split}
     M(h_i^y, h_j^b)=\alpha s_c(i, j) = \alpha \frac{M_y^T(h_i^y) M_b(h_j^b)}{ \Vert M_y(h_i^y) \Vert \Vert M_b(h_j^b) \Vert}
\end{split}
\end{equation}
where $\alpha$ is a learnable scale factor and is initialized as 1 in this paper. We term Eq.~(\ref{eq:ourmi}) with $M$ defined in Eq.~(\ref{eq:content}) as CSAD-Content.

\subsubsection{Local Structural Representation Learning} \label{sec:struc}
We further introduce local structural representation to enrich the capacity of $M$. Our intuition is that for the target representation $h^{y}_i$ of the $i$-th sample, we would like to have its bias not be guessed by its neighbors. In other words, we would like to encourage $h^y$ and $h^b$ to be different in terms of topological structure, and intuitively, such constraint can provide the clue for learning a stronger $M$ for cross-sample mutual information estimation. For simplicity, we only elaborate on how to learn the local structural feature for $h^y$, and a similar procedure applies to $h^b$. 

First, we construct an undirected graph for $h^y$ as $G^{y}=<{V}^y,{E}^{y}>$ over the training samples, here ${V}^y \in \mathbb{R}^{N}$ represents the nodes, ${E}^y \in \mathbb{R}^{N \times N}$ denotes edges, and $N$ is the number of nodes. To obtain the weights for edges ${E}^{y}$, we first  calculate the pair-wise cosine similarity among samples in $h^y$ as
\begin{equation} \label{eq:edgecal}
    s^y_{i,j}= \frac{M_y^T(h_i^y) M_y(h_j^y)}{ \Vert M_y(h_i^y) \Vert \Vert M_y(h_j^y) \Vert}, 
\end{equation}
and then obtain a normalized adjacency matrix by softmax. That is, $e^y_{i,j}$, which is the $(i,j)$-th element of ${E}^{y}$, could then be calculated as 
\begin{equation} \label{eq:normedge}
    e^y_{i,j} = \frac{\exp(\tau s^y_{i,j})}{\sum_j \exp( \tau s^y_{i,j})},
\end{equation}
where $\tau$ is a learnable scale factor and is initialized as $10$ in this paper. With the obtained graph $G^{y}$, we then apply Random Walk with Restart (RWR) to capture the local structure of each sample \cite{zhang2019adaptive,aziere2019ensemble}. Formally, RWR is conducted as:
\begin{equation}
    \label{eq:rwr}
    {r}_i^{t+1,y} = c  {{{E}^y}} {r}_i^{t,y} + (1-c) {a}_i,
\end{equation}
where the $r_i^{t,y} \in \mathbb{R}^{N}$ is the proximity between the $i$-th node to all other nodes at the $t$-th propagation, ${E}^y$ is the edge and defines the transition probability of the propagation, $c$ is the restart probability and is set to $0.5$ in this paper, and $a_i \in \mathbb{R}^{N}$ is the starting vector with the $i$-th element set to 1 and 0 for others. We start with ${r}_i^{0,y}=a_i$ and recursively perform Eq.~(\ref{eq:rwr}) until convergence. The closed-from converged solution is \cite{tong2006fast}
\begin{equation}
    \label{eq:rwr_s}
    {r}_i^{y} = (1-c)({I} - c {{{E}}}_i)^{-1} {a}_i.
\end{equation}
At last, we normalize $r_{i}^y $ as $r_{i,j}^y = \frac{r_{i,j}^y}{\sum_j r_{i,j}^y}$ to have a categorical instance-wise assignment vector. The obtained ${r}_i^{y}$ captures the local structure of $h^y_i$. Likewise, the local structural representation ${r}_i^{b}$ is computed in the same fashion for $h^b_i$.

To obtain the structural similarity $s_s(i,j)$ between $h^y_i$ and $h^b_j$, since $r_i^y$ and $r_j^b$ are categorical vectors, we calculate the inverse symmetric Cross Entropy between them as
\begin{equation}\label{eq:csadstruc}
\begin{split}
    s_s(i,j) =
    \frac{1}{2}((r_i^y)^T\log{r_j^b}) +  (r_j^b)^T\log{r_i^y})).
\end{split}
\end{equation}
We term Eq.~(\ref{eq:ourmi}) with $M$ defined in Eq.~(\ref{eq:csadstruc}), \textit{i.e.}, $M(h_i^y, h_j^b) = \alpha s_s(i, j)$, as CSAD-Struc. Note that it is infeasible to directly calculate the graph with all training data and we construct the graph $G$ with the samples in the mini-batch. 

\subsubsection{Joint Content and Structural Representation Learning} \label{sec:overall}
We jointly use the content and local structural representation for our mutual information estimator $M$ as
\begin{equation}\label{eq:joint}
    M(h_i^y, h_j^b) = \alpha (s_c(h_i^y, s^b_j) + s_s(h_i^y, s^b_j)).
\end{equation}
By jointly adopting the content and local structural representations, we could provide a more comprehensive estimation of the mutual information. We term Eq.~(\ref{eq:ourmi}) with $M$ defined in Eq.~(\ref{eq:joint}) as CSAD. Experimental results show that the joint representation outperforms either CSAD-Content or CSAD-Struc.


\subsection{Training Strategy}
We present pseudo code of CSAD in Algorithm \ref{algdfd} and please refer to Fig. (\ref{fig:dig1d}) for an illustration. We omit the pretraining stages for all components and the algorithm to train mutual information estimator, and please refer to supplemental material for more detail. The number of inner-loop iteration $K$ is set as 10 throughout the paper. The target and bias predictors are trained with cross-entropy loss in this paper but can be directly replaced with other loss functions for general tasks. We note that, from Algorithm \ref{algdfd}, the feature extractor $F$ is updated only to minimize the target prediction loss (Line 4) and to minimize the mutual information Eq.~(\ref{eq:ourmi}) (Line 14), and will never be updated to minimize the bias prediction loss and maximize the Eq.~(\ref{eq:ourmi}). Therefore, after optimized, the representation learned by the feature extractor will only be able to fulfill the target task and have little information on the bias. In our implementation, we propose a hyper-parameter $\lambda$ for Step 4 to achieve a balance between fairness and accuracy.  As the bias branch and mutual information estimator will be discarded at testing time, our method introduces no extra cost for inference. 

\begin{algorithm} 
\caption{Learning Bias-Invariant Representation} 
\label{algdfd} 
\begin{algorithmic}[1]
\REQUIRE Training data $x=\lbrace (x_i,y_i,b_i) \rbrace$;
\WHILE{not converge}
\STATE Sample a minibatch of data;
\STATE \textbf{\# STEP 1: Update Feature Extractor and Target Branch}
\STATE Update \textbf{$F$}, $D^y$, and $P^y$ to minimize the target predication loss;
\STATE \textbf{\# STEP 2: Update Bias Branch}
\FOR{k = 1, \dots, K}
\STATE Update $D^b$ and $P^b$ to minimize the bias predication loss;
\ENDFOR
\STATE \textbf{\# STEP 3: Update MI Estimator}
\FOR{k = 1, \dots, K}
\STATE Update $M$ to maximize Eq.~(\ref{eq:ourmi});
\ENDFOR
\STATE  \textbf{\# STEP 4: Adversarial Debiasing}
\STATE Update \textbf{$F$} to minimize Eq.~(\ref{eq:ourmi});
\ENDWHILE
\end{algorithmic} 
\end{algorithm}
\vspace{-3mm}

\section{Experiments}
In this section, we conduct experiments on various datasets to fully demonstrate the effectiveness of the proposed method. Experimental settings and datasets adopted by debiasing and fairness papers are different from each other, and we mainly follow three different settings as \cite{kim2019learning,bsdfhgdgfh,alvi2018turning}, \cite{nam2020learning} and \cite{yurochkin2019training,yurochkin2020sensei}. We conduct experiments on Colored MNIST~\cite{kim2019learning}, IMDB face~\cite{rothe2015dex}, CelebA~\cite{liu2015faceattributes},  mPower~\cite{bot2016mpower}, and Adult~\cite{asuncion2007uci}. For Colored MNIST, IMDB face, and mPower, we follow the debiasing setting adopted by \cite{kim2019learning,bsdfhgdgfh,alvi2018turning}, for CelebA, we follow the setting adopted by \cite{nam2020learning}, and for Adult, we follow the fairness setting adopted by \cite{yurochkin2019training,yurochkin2020sensei}. Among these datasets, Colored MNIST, IMDB Face, and CelebA are image datasets, mPower is a time series dataset, and Adult is a tabular dataset. We run all experiments three times and report the mean accuracy~\cite{bsdfhgdgfh}. We implement our method with Pytorch and all experiments are run on a Linux machine with a Nvidia GTX 1080 Ti graphic card. 
\begin{table*}[]
\centering
\caption{Evaluation results on the Colored MNIST test set. Smaller variance $\sigma^2$ on training set represents severer bias.}
\begin{tabular}{@{}lccccccc@{}}
\toprule
\multirow{2}{*}{Method} & \multicolumn{7}{c}{Color vairance $\sigma^2$}                    \\ \cmidrule(l){2-8} 
                          &0.020 & 0.025 & 0.030 & 0.035 & 0.040 & 0.045 & 0.050 \\ \midrule
Baseline                  &0.476 & 0.542 & 0.664 & 0.720 & 0.785 & 0.838 & 0.870 \\
Alvi et al. \cite{alvi2018turning}                &0.676 & 0.713 & 0.794 & 0.825 & 0.868 & 0.890 & 0.917 \\
Kim et al. \cite{kim2019learning}                &0.818 & 0.882 & 0.911 & 0.929 & 0.936 & 0.954 & 0.955 \\ 
Ruggero et al. \cite{bsdfhgdgfh}           &0.864 & 0.925 & 0.959 & 0.973 & 0.975 & 0.980 & 0.982 \\ \midrule
AD-JSD (ours)  &{0.896}  &{0.937} & {0.959} &{0.974} & {0.975}      & {0.980}      &{0.980}\\ \hdashline
CSAD-Content (ours)  &{0.933}  &{0.959} & {0.963} &{0.976} & {0.978}      & {0.980}      &{0.983}\\
CSAD-Struc (ours) &0.928 &0.955 &0.967  &0.973 &0.980 &0.981  &0.985\\

CSAD (ours) &\textbf{0.943} & \textbf{0.961} &\textbf{0.970} &\textbf{0.980} &\textbf{0.981} &\textbf{0.982} &\textbf{0.985}\\ \bottomrule
\end{tabular}
\label{tab:colormnistdata}\vspace{-3mm}
\end{table*}
\begin{table}[]
\caption{Evaluation results on the IMDB face dataset for gender prediction. All networks were evaluated with test set and the other training set.}
\centering
\small
\begin{tabular}{@{}lcc||cc@{}}
\toprule
\multirow{2}{*}{Method} & \multicolumn{2}{c}{Trained on EB1} & \multicolumn{2}{c}{Trained on EB2} \\ \cmidrule(l){2-5} 
                        & EB2              & Test            & EB1              & Test            \\ \midrule
Baseline                & 0.5986           & 0.8442          & 0.5784           & 0.6975          \\
Alvi et al. \cite{alvi2018turning}                & 0.6374           & 0.8556          & 0.5733           & 0.6990          \\
Kim et al. \cite{kim2019learning}                & 0.6800           & 0.8666          & 0.6418           & 0.7450          \\
Ruggero et al. \cite{bsdfhgdgfh}                      & 0.6840      & \textbf{0.8720}     & 0.6310       & 0.7450     \\ \midrule
CSAD (ours)                    &  \textbf{0.7038 }           & 0.8696            & \textbf{0.6811}                 &\textbf{0.7865 }            \\ \bottomrule
\end{tabular}
\label{tab:imdb}\vspace{-3mm}
\end{table}
\begin{table}[]
\centering
\caption{Evaluation results on CelebA.}
\begin{tabular}{lcc}
\hline
Mehtod         & Unbaised    & Bias-conflicting   \\ \hline
\multicolumn{3}{c}{Target attribute: \textit{BlondHair}}   \\
Baseline       & 0.7025       & 0.5252              \\
Group DRO \cite{sagawa2019distributionally}      & 0.8424       & 0.8124              \\
LfF \cite{nam2020learning}            & 0.8543       & 0.8340              \\ \midrule
CSAD (ours)     & \textbf{0.8936}       & \textbf{0.8753}              \\ \midrule
\multicolumn{3}{c}{Target attribute: \textit{HeavyMakeup}} \\ 
Baseline       & 0.6200       & 0.3375              \\
Group DRO \cite{sagawa2019distributionally}     & 0.6488       & 0.5024              \\
LfF \cite{nam2020learning}           & 0.6620        & 0.4548              \\ \midrule
CSAD (ours)     & \textbf{0.6788}       & \textbf{0.5344}             \\ \bottomrule 
\end{tabular}
\label{tab:celeba}\vspace{-3mm}
\end{table}
\begin{table}[]
\centering
\caption{Evaluation results on the mPower tapping dataset.}
\begin{tabular}{@{}lccccccc@{}}
\toprule
Method     & AUC   & AP   & F1    \\ \midrule
Baseline     & 0.735   & 0.419 & 0.553 \\
Kim et al. \cite{kim2019learning} & 0.759   & 0.424 & 0.572 \\ \midrule
CSAD (ours)       & \textbf{0.772}  & \textbf{0.434} & \textbf{0.581} \\ \bottomrule
\end{tabular}
\label{tab:mpowerresultsd}\vspace{-3mm}
\end{table}

\subsection{Colored MNIST}
The Colored MNIST dataset~\cite{kim2019learning} introduces color bias to the standard MNIST dataset~\cite{lecun2010mnist}, and the digits are class-wisely colored for the training set following~\cite{kim2019learning}. Smaller $\sigma$ means more severely biased training data. We compare CSAD with other debaising methods, including \cite{alvi2018turning}, \cite{kim2019learning} and \cite{bsdfhgdgfh}. The results of other methods are retrieved from their papers. In addition, we compare our approach with three ablation models of our methods, namely AD-JSD, CSAD-Content and CSAD-Struc. For AD-JSD, we adopt the disentanglement framework of CSAD but with $I_{JSD}$ for mutual information estimation Eq.(\ref{eq:mijsd}) instead of the proposed $I_{CS}$ Eq.(\ref{eq:ourmi}) with the content representation Eq.(\ref{eq:content}). For CSAD-content and CSAD-struc, we use content Eq.(\ref{eq:content}) or structural Eq.(\ref{eq:csadstruc}) features only for mutual information estimation. For CSAD, we adopt the proposed $I_{CS}$ with the joint features Eq.~(\ref{eq:joint}). For this dataset, $(i,j) \in \Omega$ if the difference between their color is equal to or smaller than 1 for each channel.



\begin{table*}[!htbp]
\centering
\caption{Income prediction of Adult dataset. * denote the results are reproduced by ourselves with a three-layer MLP.}
\begin{tabular}{@{}lcccccccccccc@{}}
\toprule
               & BA$\uparrow$   & S-Con$\uparrow$ & GR-Con$\uparrow$ & $\text{Gap}_G^{\text{RMS}}\downarrow$ & $\text{Gap}_R^{\text{RMS}}\downarrow$ & $\text{Gap}_G^{\text{max}}\downarrow$ & $\text{Gap}_R^{\text{max}}\downarrow$ \\ \midrule
Baseline       & \textbf{82.9} & .848  & .865   & .179      & .089      & .216      & .105      \\
Baseline*  & 82.7 & .844  & .831   & .182      & .087      & .212      & .110      \\
Project \cite{yurochkin2019training}       & 82.7 & .868  & \textbf{1.00}   & .145      & .064      & .192      & .086      \\
Adv. debiasing \cite{zhang2018mitigating} & 81.5 & .807  & .841   & .082      & .070      & .110      & .078      \\
CoCL \cite{de2019bias}           & 79.0 & -     & -      & .163      & .080     & .201      & .109      \\
SenSR \cite{yurochkin2019training}          & 78.9 & .934  & .984   & .068      & .055      & .087      & .067      \\
SenSeI \cite{yurochkin2020sensei}         & 76.8 & \textbf{.945}  & .963   & \textbf{.043}     & .054      & \textbf{.053}      & .064      \\ \midrule
CSAD (ours)           & 80.4 & {.938}      & .975       & .065     & \textbf{.042}     & .073       & \textbf{.058}      \\ \bottomrule
\end{tabular}
\label{tab:adult}
\vspace{-3mm}
\end{table*}

According to Table~\ref{tab:colormnistdata}, all variants of CSAD outperform the existing approaches with different $\sigma^2$. Notably, our model achieves even more significant improvement for severely biased datasets (smaller $\sigma^2$), demonstrating the effectiveness of the disentanglement framework and the necessity to consider cross-sample information. Moreover, the proposed $I_{CS}$ outperforms $I_{JSD}$ by comparing AD-JSD and CSAD-Content with AD-JSD, showing the superiority of the proposed mutual information estimator Eq.~(\ref{eq:ourmi}). Additionally, the proposed CSAD-Struc performs slightly better than CSAD-Content, and CSAD with the joint representation performs favourably against either CSAD-Content or CSAD-Struc, showing that 1) the structural representation could benefit the learning process for adversarial debiasing and 2) jointly considering the content and structural features could lead to better performance.


\subsection{IMDB face}
The IMDB face dataset~\cite{rothe2015dex} is a face image dataset. Following \cite{kim2019learning,bsdfhgdgfh}, the images are divided into three subsets, namely: Extreme bias 1 (EB1): women aged 0-29, men aged 40+; Extreme bias 2 (EB2): women aged 40+, men aged 0-29; Test set: 20\% of the cleaned images aged 0-29 or 40+. As a result, EB1 and EB2 are biased towards the age, since EB1 consists of younger females and older males and EB2 consists of younger males and older females. We adopt a pretrained ResNet18 \cite{he2016deep} on ImageNet \cite{deng2009imagenet} following \cite{bsdfhgdgfh,kim2019learning} as the feature extractor. Moreover, we freeze the BN layer to stabilize the training process. The $(i,j) \in \Omega$ if they share same bias label. We train our models on EB1 (EB2) and evaluate the trained model on EB2 (EB1) and the testing set. Table~\ref{tab:imdb} shows the prediction results. The biased training samples pose a serious threat to the baseline method and make the obtained model generalize poorly on unseen data. By contrast, the models obtained by debiasing methods are much more robust with a age-invariant representation, and the proposed CSAD achieves a much better performance. 


\subsection{CelebA}
CelebA dataset contains 40 attributes of face images. We follow Nam \textit{et al.} to conduct experiments on the official training (162770 samples) and validation (19867 samples) set  to respectively predict \textit{BlondHair} and \textit{HeavyMakeup} with the bias attribute as \textit{Male} \cite{nam2020learning}. To evaluate the performance, we construct unbiased set and bias-conflicting set from original validation set. Following \cite{nam2020learning}, the unbiased set is constructed with all validation data, and we report a weighted average accuracy based on the target-bias pairs. The bias-conflicting set is constructed by the data which have same target and bias values, e.g., \textit{BlondHair-Male} as there are few male with \textit{BlondHair} in the training set. $(i,j) \in \Omega$ if they share same bias label. The results are shown in Table \ref{tab:celeba}. We adopt a pretrained Resnet-18 as the feature extractor with frozen BN layer. We compare our method with Group DRO \cite{sagawa2019distributionally} and LfF \cite{nam2020learning}. Based on the results, our method outperforms other approaches on different scenarios.  



\subsection{mPower}
The mPower is collected to develop a smartphone-based remote diagnosis system for PD's patients. Subjects are required to conduct well-designed activities which could reveal the PD's symptoms. Here we conduct adversarial debiasing for the finger tapping task, where patients will tap their phones alternatively with two fingers. The mPower data has a clear bias on age, and detail statistics are provided in supplemental material. We compare our method against baseline and \cite{kim2019learning}. $(i,j) \in \Omega$ if they share same bias label. As shown in the Table \ref{tab:mpowerresultsd}, our model improves over the results of the baseline and \cite{kim2019learning}, suggesting that the learned representation of our method is more robust to the age bias. 

\subsection{Adult}
The Adult dataset is a commonly used benchmark in the algorithmic fairness literature~\cite{yurochkin2019training}. The task is to predict if a subject earns more than \$50k per year with the attributes including education, gender, race, etc. We aim to learning an income predictor that is invariant to gender and race which are protected attributes \cite{yurochkin2019training}. We first preprocess the dataset following \cite{yurochkin2019training}, and split the data into 80\% training and 20\% testing. We compare our method with advanced fairness methods including Project \cite{yurochkin2019training}, CoCL \cite{de2019bias}, Adversarial Debiasing \cite{zhang2018mitigating}, SenSR \cite{yurochkin2019training}, and SenSeI \cite{yurochkin2020sensei}. The results of compared methods are duplicated from \cite{yurochkin2020sensei}, and are obtained with a two-layer MLP for the target task. However, recall that CSAD contains three different modules for target task including $F$, $D^y$, and $P^y$, which make the two-layer MLP not applicable for our method, we thus construct a three-layer MLP which has similar performance as shown in Table \ref{tab:adult}.  $(i,j) \in \Omega$ if they share same bias label. The results reported in Table \ref{tab:adult} is averaged over 10 different train/val splits following \cite{yurochkin2019training}. We use seven evaluation metrics, and all of them are for fairness evaluation except Balanced Accuracy~(BA)~\cite{yurochkin2019training}. Refer to~\cite{yurochkin2019training} or supplemental material for detail. Overall, according to Table \ref{tab:adult}, the proposed CSAD outperforms the state-of-the-art fairness methods that are specifically trained with the fairness metric.  We note that although SenSeI \cite{yurochkin2020sensei} seems to achieve a better performance in terms of the fairness, this is at the expense of a significant balanced accuracy drop (6.1\%) and thus may be less impractical for real life applications. By contrast, CSAD obtains a state-of-the-art performance in terms of both individual and group fairness metrics with a relatively small balanced accuracy drop (2.3\%). 

\section{Conclusions}
This paper presents an adversarial debiasing method named CSAD. CSAD is built on a novel disentanglement framework composed of six learnable modules that can respectively extract target and bias features form the input. Then, we adversarially learn to mine and remove the correlation between the target and bias features, and the correlation is measurement by a cross-sample mutual information estimator. We further boost CSAD with a joint structural and content representation. At last, we show a carefully designed training strategy to obtain the debiased model. To validate the effectiveness of the proposed method, we conduct experiments on five datasets with three debiasing benchmark settings.The results show the superior performance of CSAD for various tasks. In the future, we will extend CSAD to deal with incomplete or noisy labels and investigate the explainability and the fairness-accuracy trade-off. 



{\small
\bibliographystyle{ieee_fullname}
\bibliography{debiasing}
}
\clearpage
\newpage
\begin{strip}
\begin{center}
\textbf{\large Supplemental Materials: Learning Bias-Invariant Representation by Cross-Sample Mutual Information Minimization}
\end{center}
\end{strip}
\section{Discussion on $I_{CS}$ Eq. (2)}
The proposed $I_{CS}$ is a lower bound for $I_{JSD}$ with same definition of joint distribution and product of marginal distribution. Specifically, following the definition used in $I_{CS}$,  $I_{JSD}$ could be rewritten as 
\begin{equation*}
\begin{split}
        I_{JSD}(h^y,h^b) = & \sup -E_{(i,j)\in \Omega} \log(1+\exp(-M(h^y_i, h^b_j))) \\
                          & - E_{(i,j) \notin \Omega} \log(1+\exp(M(h^y_i, h^b_j))) \\
\end{split}
\end{equation*}

\newtheorem{Theorem}{Lemma}
\begin{Theorem}
$I_{CS}$ is a lower bound for $I_{JSD}$.
\end{Theorem}
\begin{proof}
Since $-\log(1+x)$ is convex, based on Jensen's inequality, we know
\begin{equation*} 
\begin{split}
 -E_{(i,j)\in \Omega} \log(1+g_{(i,j)}) \geq  - \log(1+E_{(i,j)\in \Omega} g_{(i,j)}).
\end{split}
\end{equation*}
By replacing $g_{(i,j)}$ with $\exp(-M(h_i^y, h_j^b))$ and $\exp(M(h_i^y, h_j^b))$ respectively, and substituting them to $I_{CS}$ and $I_{JSD}$, we directly get $I_{CS} \leq I_{JSD}$. We thus complete our proof.
\end{proof}

\section{Algorithm for CSAD}
CSAD requires pretraining for target classifier, bias classifier, and mutual information estimator. We provide the complete algorithm for CSAD in Algorithm \ref{algdfdsupp}.

\begin{algorithm} 
\caption{Learning Bias-Invariant Representation} 
\label{algdfdsupp} 
\begin{algorithmic}[1]
\REQUIRE Training data $x=\lbrace (x_i,y_i,b_i) \rbrace$;
\STATE \textbf{\# STEP 1: Pretrain Feature Extractor and Target Branch}
\STATE Pretrain  \textbf{{$F$}},  $D^y$, and $P^y$ by minimizing the target prediction loss until convergence; 
\STATE \textbf{\# STEP 2: Pretrain Bias Branch}
\STATE Pretrain  $D^b$ and $P^b$ by minimizing the bias prediction loss until convergence; 
\STATE \textbf{\# STEP 3: Pretrain Mutual Information Estimator}
\STATE Pretrain $M$ to maximize Eq. (2) until convergence;
\STATE \textbf{\# STEP 4: Iteratively Update}
\WHILE{not converge}
\STATE Sample a minibatch of data;
\STATE Update \textbf{$F$}, $D^y$, and $P^y$ to minimize the target predication loss;
\FOR{k = 1, \dots, K}
\STATE Update $D^b$ and $P^b$ to minimize the bias predication loss;
\ENDFOR
\FOR{k = 1, \dots, K}
\STATE Update $M$ to maximize Eq. (2);
\ENDFOR
\STATE  \textbf{\# Adversarial Debiasing}
\STATE Update \textbf{$F$} to minimize Eq. (2);
\ENDWHILE
\end{algorithmic} 
\end{algorithm}
\section{Algorithm for Optimizing Eq. (2)}
We provide the details on optimizing Eq. (2) with joint content and local structural learning in Algorithm \ref{algdfd2}.
\begin{algorithm} 
\caption{Cross sample Mutual Information Estimator $M$} 
\label{algdfd2} 
\begin{algorithmic}[1]
\REQUIRE Target representation $h^y=\lbrace h_i^y  \rbrace$, Bias representation $h^b=\lbrace h_i^b  \rbrace$;

\STATE {\textbf{STEP I: Content Similarity Learning}}
\STATE Calculate content similarity as Eq. (3)  
\STATE \textbf{STEP II: Structural Similarity Learning}
\STATE Calculate the pairwise similarity matrices for $h^y$ and $h^b$ by Eq. (4), and then normalize the matrices by Eq. (5) to obtain the edge $E^y$ and $E^b$.
\STATE Conduct RWR on the obtained graph $G^y$ and $G^b$ respectively to obtain $r_i^y$ and $r_i^b$ for the $i$-th sample by Eq. (7);
\STATE Normalize the $r_i^b$ and $r_i^y$ for the $i$-th sample;
\STATE Calculate the structural similarity $s_s^y(i,j)$ by Eq. (8);
\STATE \textbf{STEP III: Joint Similarity Learning}
\STATE Obtain the joint similarity by Eq. (9);
\STATE \textbf{STEP IV: Cross Sample Mutual Information Estimation}
\STATE Update $M$ to maximize Eq. (2).
\end{algorithmic} 
\end{algorithm}

\section{Dataset and Network Structure}
We provide detail structure of used network for different datasets. We omit the activation function (ReLU Layer) of the network for convenience. 

\subsection{Colored Mnist}

The Colored MNIST dataset~\cite{kim2019learning} introduces color bias to the standard MNIST dataset~\cite{lecun2010mnist}, and the digits are class-wisely colored for the training set following~\cite{kim2019learning}. We assign a mean color for each class of digit. Then, for each training image, its color is sampled from a normal distribution with the mean set as the class-wise mean color and a predefined variance $\sigma^2$. We vary the variance $\sigma^2$ from $0.02$ to $0.05$ to have a different amount of bias in the training data, and smaller $\sigma^2$ represents more color bias. To have a bias-free testing set, testing images are generated similarly to the training ones but with the mean randomly sampled from ten mean colors. The color label is grouped into eight different categories for each RGB channel following \cite{kim2019learning}. 

The used network for target and bias task follows \cite{bsdfhgdgfh}. For detail we adopt two convolutional layers with kernel size as 5 and 64 filters as feature extractor. The disentangler is implemented with a fully connected layer (1024-128). Class Predictor is implemented with two fully connected layer as (128-64-10). Bias  Predictor is also implemented with two fully connected layer for each color as (128-64-8). The mutual information estimator is a three-layer fully connected network as (128-64-32-32).

\begin{figure}
\centering
\includegraphics[width=0.45\textwidth]{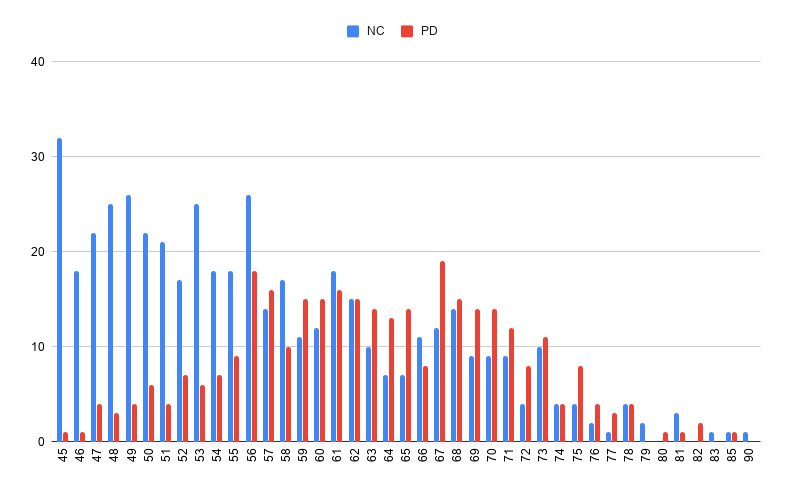}
\caption{The age distribution of PD and non-PD groups for the mPower dataset. The dataset is biased on age, as PD patients are typically elder.}
\label{fig:mpowertappigndf}
\end{figure}
\subsection{IMDB face}
The IMDB face dataset~\cite{rothe2015dex} is a face image dataset that contains $460,723$ face images of $20,284$ celebrities with their information regarding age and gender. Following~\cite{bsdfhgdgfh}, a pretrained network on Image~\cite{hassner2015age} for age and gender annotation is used to filter out the misannotated images, resulting a cleaned dataset of 112,340 samples. We aim to conduct gender prediction on an age-biased training set. Likewise, the cleaned images are divided into three subsets, namely: Extreme bias 1 (EB1): women aged 0-29, men aged 40+; Extreme bias 2 (EB2): women aged 40+, men aged 0-29; Test set: 20\% of the cleaned images aged 0-29 or 40+. As a result, EB1 and EB2 are biased towards the age, since EB1 consists of younger females and older males and EB2 consists of younger males and older females.

Follow \cite{kim2019learning}, feature extractor is implemented with a pretrained ResNet-18 by modifying the last fc layer as (512-256). The disentangler is a fully connected layer as (256-64). The class predictor is a fully connected layer as (64-1) and the bias predictor is also a fully connected layer as (64-12). The mutual information estimator is a three-layer fully connected network as (64-64-32-32).

\subsection{CelebA}

Follow \cite{nam2020learning}, feature extractor is implemented with a pretrained ResNet-18 by modifying the last fc layer as (512-256). The disentangler is a fully connected layer as (256-64). The class predictor is a fully connected layer as (64-1) and the bias predictor is also a fully connected layer as (64-1). The mutual information estimator is a three-layer fully connected network as (64-64-32-32). We train the model with a balanced batch to alleviate the unbalance problem.

\subsection{mPower}
We illustrate the age bias for mPower dataset in Fig. \ref{fig:mpowertappigndf}. We can observe that most PD's patients are the elder, and it is thus necessary to remove the age bias for PD's diagnosis. We conduct adversarial debiasing for the finger tapping task, where patients will tap their phones alternatively with two fingers. To evaluate the debiasing methods, we contrive a bias-free testing set following the settings of Colored MNIST and IMDB face. For detail, we divide the age into 6 different intervals $\lbrace 45-49, 50-54, 55-59, 60-67, 65-69, 70+\rbrace$, and then draw 30 PD and NC subjects from each interval as the testing set (360 samples in total). The training set contains all the other 1044 patients. 

Feature extractor is implemented with a 6-layer TCN with kernel size as 5 and 64 filters. The disentangler is a fully connected layer as (64-64). The class predictor is a fully connected layer as (64-1) and the bias predictor is also a fully connected layer as (64-8). The mutual information estimator is a three-layer fully connected network as (64-64-32-32).

\subsection{Adult}
To comprehensively evaluate the performance, various metrics have been applied following \cite{yurochkin2019training}. First, Balanced accuracy (BA) is used for imbalance data. Moreover, we evaluate the model with counterfactual samples by flipping the attributes of spouse (gender\&Race) for testing records, and calculate spouse (gender\&Race) consistency S-Con (GR-Con) by the predication consistency between original and altered samples \cite{yurochkin2019training}. We also report group fair metrics provided by AIF360 \cite{aif360} with respect to race or gender, including $\text{Gap}_G^{\text{RMS}}$, $\text{Gap}_R^{\text{RMS}}$, $\text{Gap}_G^{\text{max}}$, and $\text{Gap}_R^{\text{max}}$,
and please refer to \cite{yurochkin2019training} for their definitions.

The baseline used by other methods is a two-layer MLP as (41-100-2) \cite{yurochkin2019training}. We adopt a three-layer MLP as (41-64-32-2). We note that the three-layer MLP has fewer parameters than the two-layer baseline and achieves competitive performance. For the proposed methods, we discompose the three-layer baseline into three modules. Feature extractor is a fully connected layer as (41-64). The disentangler is a fully connected layer as (64-32). The class predictor is a fully connected layer as (32-1) and the bias predictor is also a fully connected layer as (32-2). The mutual information estimator is a two-layer fully connected network as (32-32-32). We construct a balanced minibatch for Adult which contains same number of samples from each target class following \cite{yurochkin2019training}.

\section{Implementation Details}
We use Adam to train our model~\cite{kingma2014adam}. We set $K=10$, $\tau=10$ Eq.(5), $c=0.5$ Eq. (6), and $\alpha=1$ Eq. (9), and search $\lambda \in \lbrace 0.1, 0.5,  1, 10 \rbrace$.
We would like to emphasize that $\lambda$ is particularly important for our method to achieve a fairness-accuracy balance. Specifically, over large $\lambda$ would force the feature extractor to learn little information, and small ones would lead to little influence on the pretrained target classifier. We conduct experiments on Adult to show the trade-off for our method in Table \ref{tab:fairacc}. The $\lambda$ is default set to 10 for Adult. As we increase $\lambda$, our model focuses more on fairness with reduced accuracy. We note that the fairness-accuracy trade-off is still an open and significant problem for debiasing and fairness~\cite{liu2020accuracy,dutta2020there}. We will study the problem comprehensively in the future.

\begin{table}[!htbp]
\caption{We vary the $\lambda$ to study the accuracy-fairness trade-off of our method on Adult dataset. BA is balanced accuracy, while the other four metrics are used to evaluate the fairness. }
\resizebox{1\columnwidth}{!}{
\begin{tabular}{@{}lccccc@{}}
\toprule
$\lambda$        & BA $\uparrow$ & $\text{Gap}_G^{\text{RMS}}\downarrow$ & $\text{Gap}_R^{\text{RMS}}\downarrow$ & $\text{Gap}_G^{\text{max}}\downarrow$ & $\text{Gap}_R^{\text{max}}\downarrow$ \\ \midrule
1          & \textbf{81.4} & .118    & .061    & .130    & .067    \\
2          & 80.7 & .080    & .053    & .109    & .055    \\
\textbf{10} & 80.4 & .060    & .042    & .066    & .058    \\
20          & 78.9 & \textbf{.058}    & .035    & \textbf{.065}    & .050    \\
40          & 78.5 & .063    & \textbf{.030}    & .088    & \textbf{.042}    \\ \bottomrule
\end{tabular}
}
\label{tab:fairacc}
\end{table}

\end{document}